\documentclass{llncs}

\usepackage{times}
\usepackage{amsmath, amssymb, mathrsfs}  
\usepackage{graphicx}                    
\usepackage{color}
\usepackage[shortend,algo2e]{algorithm2e} 
\usepackage[section]{algorithm}                  
\usepackage{subfigure}                   
\usepackage{url}                
\usepackage{multirow}                    
\usepackage{rotating}                    
\usepackage{tikz}                        
\usepackage{gastex}                      
\usepackage{verbatim}
\usepackage[normalem]{ulem}              
\usepackage{makeidx}                     
\usepackage{arydshln}                    

\usepackage{lipsum}
\usepackage{ntheorem}

\newtheorem*{theorem*}{Theorem}
\newtheorem*{proposition*}{Proposition}

\numberwithin{equation}{section}
\numberwithin{figure}{section}
\numberwithin{table}{section}

\newcommand{\defterm}[1]{\emph{#1}}             

\newcommand{\N}{\mathbb{N}}                      
\newcommand{\class}[1]{\mathscr{#1}}             

\newcommand{\tuple}[1]{\langle #1 \rangle}       
\newcommand{\Var}{\mathit{Var}}                  

\newcommand{\SAT}{\mathsf{SAT}}
\newcommand{\UNSAT}{\mathsf{UNSAT}}

\newcommand{\MUS}{\mathsf{MUS}}
\newcommand{\MSS}{\mathsf{MSS}}

\newcommand{\MCS}{\mathsf{MCS}}

\newcommand{\resolve}{\otimes}

\newcommand{\bce}{\mathsf{BCE}}

\newcommand{\bve}{\mathsf{BVE}}

\newcommand{\sub}{\mathsf{SUB}}
\newcommand{\e}{\mathsf{E}}
\newcommand{\p}{\mathsf{P}}

\newcommand{\sbce}{\mathsf{bce}}
\newcommand{\sve}{\mathsf{ve}}
\newcommand{\sbve}{\mathsf{bve}}
\newcommand{\sssr}{\mathsf{ssr}}
\newcommand{\ssub}{\mathsf{sub}}

\newcommand{\tool}[1]{\texttt{#1}}       

\newcommand{\LMCS}{\mathsf{LMCS}}

\definecolor{midgrey}{rgb}{0.5,0.5,0.5}
\definecolor{darkred}{rgb}{0.7,0.1,0.1}

\newcommand{\mxsat}{MaxSAT\xspace}
\newcommand{\mxsmt}{MaxSMT\xspace}

\title{SAT-based Preprocessing for MaxSAT\\
(extended version)
\thanks{This is an extended version of the paper accepted for publication
in proceedings of 19-th International Conference on Logic for Programming, 
Artificial Intelligence, and Reasoning (LPAR-19), 2013. This version includes
all proofs omitted in the original paper due to space limitations.}
\thanks{
This work is partially supported by SFI~PI~grant~BEA{\-}CON
(09{\-}/IN.1/I2618), FCT grants ATTEST (CMU-PT/ELE/0009/2009) and
POLARIS (PTDC{\-}/EIA-CCO/123051/2010), and INESC-ID’s multiannual
PIDDAC funding PEst-OE{\-}/EE{\-}I/LA{\-}0021/2011.
}
}

\author
{
Anton Belov\inst{1}
\and
Ant\'{o}nio Morgado\inst{2}
\and 
Joao Marques-Silva\inst{1,2}
}

\institute{
Complex and Adaptive Systems Laboratory
University College Dublin
\and
IST/INESC-ID, Technical University of Lisbon, Portugal
}

\begin{document}

\maketitle

\looseness=-2

%

\begin{abstract}
State-of-the-art algorithms for industrial instances of MaxSAT problem rely on iterative calls to a SAT solver.
Preprocessing is crucial for the acceleration of SAT solving, and the key preprocessing techniques rely on the application of resolution and subsumption elimination. Additionally, satisfiability-preserving clause elimination procedures are often used.
Since MaxSAT computation typically involves a large number of SAT calls, we are interested in whether an input instance to a MaxSAT problem can be preprocessed \emph{up-front}, i.e. prior to running the MaxSAT solver, rather than (or, in addition to) during each iterative SAT solver call. 
The key requirement in this setting is that the preprocessing has to be \emph{sound}, i.e. so that the solution can be reconstructed correctly and efficiently after the execution of a MaxSAT algorithm on the preprocessed instance.
While, as we demonstrate in this paper, certain clause elimination procedures are sound for MaxSAT, it is well-known that this is not the case for resolution and subsumption elimination.
In this paper we show how to adapt these preprocessing techniques to MaxSAT. To achieve this we recast the MaxSAT problem in a recently introduced labelled-CNF framework, and show that within the framework the preprocessing techniques can be applied soundly. 
Furthermore, we show that MaxSAT algorithms restated in the framework have a natural implementation on top of an \emph{incremental} SAT solver. We evaluate the prototype implementation of a MaxSAT algorithm WMSU1 in this setting, demonstrate the effectiveness of preprocessing, and show overall improvement with respect to non-incremental versions of the algorithm on some classes of problems.
\end{abstract}


\section{Introduction}\label{s:intro}

Maximum Satisfiability (\mxsat) and its generalization to the case of
Satisfiability Modulo Theories (\mxsmt) find a growing number of
practical applications~\cite{manya-hdbk09,mhlpms-cj13}.
%
For problem instances originating from practical applications, state
of the art \mxsat algorithms rely on iterative calls to a SAT oracle.
Moreover, and for a growing number of iterative algorithms, the calls
to the SAT oracle are guided by iteratively computed unsatisfiable 
cores~(e.g.\ \cite{mhlpms-cj13}). 

In practical SAT solving, formula preprocessing has been extensively
studied and is now widely accepted to be an often effective, if 
not crucial, technique.
In contrast, formula preprocessing is not used in practical \mxsat
solving. Indeed, it is well-known that resolution and subsumption elimination,
which form the core of many effective preprocessors, are
unsound for \mxsat solving~\cite{manya-hdbk09}. This has been
addressed by the development of a resolution calculus specific to
\mxsat~\cite{manya-aij07}. 
Nevertheless, for practical instances of \mxsat, dedicated \mxsat
resolution is ineffective.

The application of SAT preprocessing to problems where a SAT oracle
is used a number of times has been the subject of recent
interest~\cite{DBLP:conf/tacas/BelovJM13}.
For iterative \mxsat solving, SAT preprocessing can be used internally
to the SAT solver. 
However, we are interested in the question of whether an input instance of a MaxSAT problem can be preprocessed \emph{up-front}, i.e. prior to running the MaxSAT solver, rather than (or, in addition to) during each iterative SAT solver call. 
The key requirement in this setting is that the preprocessing has to be \emph{sound}, i.e. so that the solution can be reconstructed correctly and efficiently after the execution of a MaxSAT algorithm on the preprocessed instance.

In this paper we make the following contributions. 
First, we establish that certain class of clause elimination procedures, and in particular monotone clause elimination procedures such as blocked clause elimination \cite{JarvisaloBH:TACAS2010}, are sound for MaxSAT.
Second, we use a recently proposed labelled-CNF framework \cite{DBLP:journals/corr/abs-1207-1257,DBLP:conf/tacas/BelovJM13} to re-formulate \mxsat and its generalizations, and show that within the framework the resolution and subsumption-elimination based preprocessing techniques can be applied soundly. This result complements a similar result with respect to the MUS computation problem presented in \cite{DBLP:conf/tacas/BelovJM13}.
An interesting related result is that \mxsat algorithms formulated in
the labelled-CNF framework can naturally implemented on top of an \emph{incremental} SAT solver (cf. \cite{DBLP:journals/entcs/EenS03}). 
We evaluate a prototype implementation of a MaxSAT algorithm WMSU1 \cite{DBLP:conf/sat/FuM06,DBLP:conf/sat/AnsoteguiBL09,DBLP:conf/sat/ManquinhoSP09} in this setting, demonstrate the effectiveness of preprocessing, and show overall improvement with respect to non-incremental versions of this algorithm on weighted partial MaxSAT instances.

\section{Preliminaries}\label{s:prelim}

We assume the familiarity with propositional logic, its clausal fragment, SAT solving in general, and the assumption-based incremental SAT solving cf.~\cite{DBLP:journals/entcs/EenS03}.
We focus on formulas in CNF (\defterm{formulas}, from hence on), which we treat as (finite) (multi-)sets of clauses. When it is convenient we treat clauses as sets of literals, and hence we assume that clauses do not contain duplicate literals. Given a formula $F$ we denote the set of variables that occur in $F$ by $\Var(F)$, and the set of variables that occur in a clause $C \in F$ by $\Var(C)$. 
An \defterm{assignment} $\tau$ for $F$ is a map $\tau: \Var(F) \to \{0,1\}$. Assignments are extended to formulas according to the semantics of classical propositional logic. If $\tau(F) = 1$, then $\tau$ is a \defterm{model} of $F$. If a formula $F$ has (resp. does not have) a model, then $F$ is \defterm{satisfiable} (resp. \defterm{unsatisfiable}). By $\SAT$ (resp. $\UNSAT$) we denote the set of all satisfiable (resp. unsatisfiable) CNF formulas.

\vspace{-10pt}
\subsubsection{MUSes, MSSes, and MCSes}

Let $F$ be an unsatisfiable CNF formula.
A formula $M \subseteq F$ is a \defterm{minimal unsatisfiable subformula (MUS)} of $F$ if $(i)$~$M \in \UNSAT$, and $(ii)$~$\forall C \in M$,  $M \setminus \{C\} \in \SAT$. The set of MUSes of $F$ is denoted by $\MUS(F)$.
Dually, a formula $S \subseteq F$ is a \defterm{maximal satisfiable subformula (MSS)} of $F$ if $(i)$~$S \in \SAT$, and $(ii)$~$\forall C \in F \setminus S$, $S \cup \{ C \} \in \UNSAT$. The set of MSSes of $F$ is denoted by $\MSS(F)$.
Finally, a formula $R \subseteq F$ is a \defterm{minimal correction subset (MCS), or, co-MSS} of $F$, if $F \setminus R \in \MSS(F)$, or, explicitly, if $(i)$~$F \setminus R \in \SAT$, and $(ii)$~$\forall C \in R$, $(F \setminus R) \cup \{ C \} \in \UNSAT$. Again, the set of MCSes of $F$ is denoted by $\MCS(F)$. 
The MUSes, MSSes and MCSes of a given unsatisfiable formula $F$ are connected via so-called \defterm{hitting sets duality} theorem, first proved in \cite{reiter-aij87}. 
The theorem states that $M$ is an MUS of $F$ if and only if $M$ is an irreducible hitting set\footnote{For a given collection $\class{S}$ of arbitrary sets, a set $H$ is called a \defterm{hitting set} of $\class{S}$ if for all $S \in \class{S}$, $H \cap S \ne \emptyset$. A hitting set $H$ is \defterm{irreducible}, if no $H' \subset H$ is a hitting set of $\class{S}$. Irreducible hitting sets are also known as hypergraph transversals.} of the set $\MCS(F)$, and vice versa: $R \in \MCS(F)$ iff $R$ is an irreducible hitting set of $\MUS(F)$.

\vspace{-10pt}
\subsubsection{Maximum satisfiability}

A \defterm{weighted clause} is a pair $(C, w)$, where $C$ is a clause, and $w \in \N^{+} \cup \{ \top \}$ is the cost of falsifying $C$. The special value $\top$ signifies that $C$ \defterm{must} be satisfied, and $(C, \top)$ is then called a \defterm{hard} clause, while $(C, w)$ for $w \in \N^{+}$ is called a \defterm{soft} clause. A \defterm{weighted CNF (WCNF)} is a set of weighted clauses, $F = F^H \cup F^S$, where $F^H$ is the set of hard clauses, and $F^S$ is the set of soft clauses. The satisfiability, and the related concepts, are defined for weighted CNFs by disregarding the weights.
For a given WCNF $F = F^H \cup F^S$, a \defterm{MaxSAT model} for $F$ is an assignment $\tau$ for $F$ that satisfies $F^H$. A \defterm{cost} of a MaxSAT model $\tau$, $cost(\tau)$, is the sum of the weights of the soft clauses \defterm{falsified} by $\tau$. For the rest of this paper, we assume that $(i)$ $F^H \in \SAT$, i.e.~$F$ has at least one MaxSAT model, and $(ii)$ $F \in \UNSAT$, i.e.~$cost(\tau) > 0$. 
\defterm{(Weighted) (Partial) MaxSAT} is a problem of finding a MaxSAT model of the minimum cost for a given WCNF formula $F = F^H \cup F^S$. The word ``weighted'' is used when there are soft clauses with weight $>1$, while the word ``partial'' is used when $F^H \ne \emptyset$.

A straightforward, but nevertheless important, observation is that solving a weighted partial MaxSAT problem for WCNF $F$ is equivalent to finding a minimum-cost MCS $R_{min}$ of $F$, or, alternatively, a minimum-cost hitting set of $\MUS(F)$\footnote{For a set of weighted clauses, its cost is the sum of their weights, or $\top$ if any of them is hard.}. The MaxSAT solution is then a model for the corresponding MSS of $F$,~i.e. $F \setminus R_{min}$.

\vspace{-10pt}
\subsubsection{SAT preprocessing}

Given a CNF formula $F$, the goal of preprocessing for SAT solving is to compute a formula $F'$ that is equisatisfiable with $F$, and that might be easier to solve. The computation of $F'$ and a model of $F$ from a model of $F'$ in case $F' \in \SAT$, is expected to be fast enough to make it worthwhile for the overall SAT solving.
Many SAT preprocessing techniques rely on a combination of resolution-based preprocessing and clause-elimination procedures. Resolution-based preprocessing relies on the application of the resolution rule to \emph{modify} the clauses of the input formula and/or to reduce the total size of the formula. Clause-elimination procedures, on the other hand, do not change the clauses of the input formula, but rather remove some of its clauses, producing a subformula the input formula.
SAT preprocessing techniques can be described as non-deterministic procedures that apply atomic preprocessing steps to the, initially input, formula until a fixpoint, or until resource limits are exceeded. 

One of the most successful and widely used SAT preprocessors is the SatElite preprocessor presented 
in~\cite{een-sat05}. 
The techniques employed by SatElite are: bounded variable elimination (BVE), subsumption elimination, self-subsuming resolution (SSR), and, of course, unit propagation (UP). An additional practically relevant preprocessing technique is blocked clause elimination (BCE) \cite{JarvisaloBH:TACAS2010}. We describe these techniques below, as these will be discussed in this paper in the context of MaxSAT. 

\defterm{Bounded variable elimination (BVE)} 
\cite{een-sat05} 
is a resolution-based preprocessing technique, rooted in the original Davis-Putnam algorithm for SAT. Recall that for two clauses $C_1 = (x \lor A)$ and $C_2 = (\neg x \lor B)$ the \defterm{resolvent} $C_1 \resolve_x C_2$ is the clause $(A \lor B)$.  
For two sets $F_x$ and $F_{\neg x}$ of clauses that all contain the literal $x$ and $\neg x$, resp., define  
$F_x \resolve_x F_{\neg x} = \{C_1 \resolve_x C_2 \mid C_1 \in F_x, C_2 \in F_{\neg x}, \mbox{ and } C_1 \resolve_x C_2 \mbox{ is not a tautology}\}.$
The formula $\sve(F,x) = F \setminus (F_x \cup F_{\neg x}) \cup (F_x \resolve_x F_{\neg x})$ is equisatisfiable with $F$, however, in general, might be quadratic in the size of $F$. Thus the atomic operation of \emph{bounded} variable elimination is defined as $\sbve(F,x) = {\bf if\ } (|\sve(F,x)| < |F|) {\bf\ then\ } \sve(F,x) {\bf\ else\ } F$. A formula $\bve(F)$ is obtained by applying $\sbve(F,x)$ to all variables in $F$\footnote{Specific implementations often impose additional restrictions on $\bve$.}.

\defterm{Subsumption elimination (SE)} is an example of a clause elimination technique. A clause $C_1$ \defterm{subsumes} a clause $C_2$, if $C_1 \subset C_2$. For $C_1, C_2 \in F$, define 
$\ssub(F,C_1,C_2) = {\bf if\ } (C_1 \subset C_2) {\bf\ then\ } F \setminus \{ C_2 \} {\bf\ else\ } F$.
The formula $\sub(F)$ is then obtained by applying $\ssub(F,C_1,C_2)$ to all clauses of $F$.

Notice that \defterm{unit propagation (UP)} of a unit clause $(l) \in F$ is just an application of $\ssub(F,(l),C)$ until fixpoint (to remove satisfied clauses), followed by $\sbve(F,var(l))$ (to remove the clause $(l)$ and the literal $\neg l$ from the remaining clauses), and so we will not discuss UP explicitly.

\defterm{Self-Subsuming resolution (SSR)} uses resolution and subsumption elimination. Given two clauses $C_1 = (l \lor A)$ and $C_2 = ({\neg l} \lor B)$ in $F$, such that $A \subset B$, we have $C_1 \resolve_l C_2 = B \subset C_2$, and so $C_2$ can be replaced with $B$, or, in other words, $\neg l$ is removed from $C_2$. Hence, the atomic step of SSR, $\sssr(F,C_1,C_2)$, results in the formula $F \setminus \{ C_2 \} \cup \{ B \}$ if $C_1,C_2$ are as above, and $F$, otherwise.

An atomic step of \defterm{blocked clause elimination (BCE)} consists of removing one blocked clause --- a clause $C \in F$ is \defterm{blocked} in $F$ \cite{DBLP:journals/dam/Kullmann99}, if for some literal $l \in C$, every resolvent of $C$ with $C' \in F$ on $l$ is tautological. A formula $\bce(F)$ is obtained by applying $\sbce(F,C) = {\bf if\ } (C \text{ blocked in } F) {\bf\ then\ } F \setminus \{ C \} {\bf\ else\ } F$ to all clauses of $F$. 
Notice, that a clause with a pure literal is blocked (vacuously), and so pure literal elimination is a special case of BCE.
BCE possesses an important property called \defterm{monotonicity}: for any $F' \subseteq F$, $\bce(F') \subseteq \bce(F)$. This holds because if $C$ is blocked w.r.t. to $F$, it will be also blocked w.r.t to any subset of $F$. Notice that subsumption elimination is \emph{not} monotone.

\section{SAT preprocessing and MaxSAT}\label{s:premax}

Let $F'$ denote the result of the application of one or more of the SAT preprocessing techniques, such as those discussed in the previous section, to a CNF formula $F$. The question that we would like to address in this paper is whether it is possible to solve a MaxSAT problem for $F'$, instead of $F$, in such a way that from any MaxSAT solution of $F'$, a MaxSAT solution of $F$ can be reconstructed feasibly. In a more general setting, $F$ might be a WCNF formula, and $F'$ is the set of weighted clauses obtained by preprocessing the clauses of $F$, and perhaps, adjusting their weights in some manner. The preprocessing techniques for which the answer to this question is ``yes'' will be refereed to as \defterm{sound for MaxSAT}. To be specific:
\begin{definition}\label{def:sound}
A preprocessing technique $\p$ is \defterm{sound for MaxSAT} if there exist a polytime computable function $\alpha_{\p}$ such that for any WCNF formula $F$ and any MaxSAT solution $\tau$ of $\p(F)$, $\alpha_{\p}(\tau)$ is a MaxSAT solution of $F$.
\end{definition}
This line of research is motivated by the fact that most of the efficient algorithms for industrial MaxSAT problems are based on iterative invocations of a SAT solver. Thus, if $F'$ is indeed easier to solve than $F$ by a SAT solver, it might be the case that it is also easier to solve by a SAT-based MaxSAT solver.
To illustrate that the question is not trivial, consider the following example.

\vspace{-5pt}
\begin{example}\label{ex:1} In the plain MaxSAT setting, let $F = \{ C_1, \dots, C_6 \}$, with $C_1 = (p)$, $C_2 = (\neg p)$, $C_3 = (p \lor q)$, $C_4 = (p \lor \neg q)$, $C_5 = (r)$, and $C_6 = (\neg r)$.
The clauses $C_3$ and $C_4$ are subsumed by $C_1$, and so $\sub(F) = \{ C_1, C_2, C_5, C_6 \}$. $\sub(F)$ has MaxSAT solutions in which $p$ is assigned to 0,~e.g. $\{ \tuple{p,0}, \tuple{r,0} \}$, while $F$ does not. Furthermore, $\bve(F) = \{ \emptyset \}$ --- a formula with 8 MaxSAT solutions (w.r.t. to the variables of $F$) with cost 1. $F$, on the other hand, has 4 MaxSAT solutions with cost 2. 
\end{example}

\vspace{-5pt}
Thus, even a seemingly benign subsumption elimination already causes problems for MaxSAT. While we do not prove that the technique is not sound for MaxSAT, a strong indication that this might be the case is that $\sub$ might remove clauses that are included in one or more of the MUSes of the input formula $F$ (c.f.~Example~\ref{ex:1}), and thus lose the information required to compute the MaxSAT solution of $F$. 
The problems with the application of the resolution rule in the context of MaxSAT has been pointed out already in \cite{manya-hdbk09}, and where the motivation for the introduction of the so-called \defterm{MaxSAT resolution} rule \cite{manya-aij07} and a complete proof procedure for MaxSAT based on it. However, MaxSAT resolution does not lead to effective preprocessing techniques for industrial MaxSAT since it often introduces a large number of auxiliary ``compensation'' clauses.
Once again, we do not claim that resolution is unsound for MaxSAT, but it is likely to be the case, since for example $\sve$ ran to completion on any unsatisfiable formula will always produce a formula $\{ \emptyset \}$.

In this paper we propose an alternative solution, which will be discussed shortly. But first, we observe that \emph{monotone} clause elimination procedures \emph{are} sound for MaxSAT.

\vspace{-10pt}
\subsection{Monotone clause elimination procedures}

Recall that given a CNF formula $F$, an application of clause elimination procedure $\e$ produces a formula $\e(F) \subseteq F$ equisatisfiable with $F$. Monotonicity implies that for any $F' \subseteq F$, $\e(F') \subseteq \e(F)$. Some examples of monotone clause elimination procedures include BCE (and as a special case, pure literal elimination), and also \defterm{covered clause elimination} introduced in \cite{HeuleJB:LPAR2010short}.

It was observed already in \cite{DBLP:conf/sat/KullmannLM06} that if a clause $C \in F$ is blocked in $F$, then none of the MUSes of $F$ can include $C$. Thus, $\MUS(\bce(F)) = \MUS(F)$, and therefore, by the hitting-sets duality, $\MCS(\bce(F)) = \MCS(F)$. In particular, any minimum-cost MCS of $\bce(F)$ is also a minimum-cost MCS of $F$. Thus, the \emph{cost} of any MaxSAT solution $\tau$ of $\bce(F)$ is exactly the same as of any MaxSAT solution of $F$, and moreover, there exist a MaxSAT solution of $F$ that falsifies the exact same set of clauses as $\tau$ in $\bce(F)$. The only question is whether a solution of $F$ can be feasibly constructed from $\tau$. A linear time procedure for reconstruction of satisfying assignments after BCE has been described in \cite{DBLP:conf/sat/JarvisaloB10} (Prop.~3). We show that the same procedure can be applied to reconstruct the solutions in the context of MaxSAT. We generalize the discussion to include some of the clause elimination procedures beside BCE.

\begin{definition} A clause elimination procedure $\e$ is \defterm{MUS-preserving} if $\MUS(\e(F)) = \MUS(F)$. 
\end{definition}

\begin{theorem}\label{th:mon} Any MUS-preserving clause elimination procedure is sound for MaxSAT.
\end{theorem}
\begin{proof}
Let $\e$ be an MUS-preserving clause elimination procedure, and let $\alpha_{\e}$ be a feasibly computable function that for any CNF formula $G$ maps a model of $\e(G)$ to a model of $G$ when $\e(G)$ is satisfiable. Let $F$ be a WCNF formula, and let $\tau$ be a MaxSAT solution of the formula $\e(F)$. Let $\e(F) = R \uplus S$\footnote{The symbol $\uplus$ refers to a \emph{disjoint} union.}, where $R$ (resp.~$S$) is the set of clauses falsified (resp.~satisfied) by $\tau$,~i.e. $R$ is a minimum-cost MCS of $\e(F)$, and $S$ is the corresponding MSS of $\e(F)$. Since $\e$ is MUS-preserving, $\MUS(\e(F)) = \MUS(F)$, and, by hitting-sets duality, $\MCS(\e(F)) = \MCS(F)$, and so $R$ is also a minimum-cost MCS of $F$. To show that $\tau' = \alpha_{\e}(\tau)$ satisfies $S' = F \setminus R$, we observe that since $F = R \uplus S'$, $\e(F) = \e(R \uplus S') = R \uplus \e(S')$, because $R \subset \e(F)$. Hence $S = \e(S')$, and therefore given any model $\tau$ of $S$, $\alpha_{\e}(\tau)$ is a model of $S'$.\qed
\end{proof}

\begin{proposition}\label{pr:mon} Any monotone clause elimination procedure is MUS-preserving.
\end{proposition}
\begin{proof} 
Let $\e$ be a monotone clause elimination procedure.
Clearly, any MUS of $\e(F)$ is an MUS of $F$, since $\e(F) \subseteq F$, regardless of whether $\e$ is monotone or not. Let now $M \subseteq F$ be any MUS of $F$. Since $M \in \UNSAT$, we have $\e(M) \in \UNSAT$, because $\e$ is preserves satisfiability. On the other hand, $\e(M) \subseteq M$ and so we must have $\e(M) = M$, because $M$ is an MUS. By monotonicity, and since $M \subseteq F$, we have $\e(M) \subseteq \e(F)$, and so $M = \e(M) \in \MUS(\e(F))$.\qed
\end{proof}   

\begin{corollary} Any monotone clause elimination procedure is sound for MaxSAT.
\end{corollary}
 
\vspace{-10pt}
\subsection{Resolution-based and subsumption elimination based techniques}

To enable sound preprocessing for MaxSAT using resolution-based and subsumption elimination based preprocessing techniques, we propose to recast the MaxSAT problem in the framework of so-called \defterm{labelled CNF (LCNF)} formulas. The framework was introduced in \cite{DBLP:journals/corr/abs-1207-1257}, and was already used to enable sound preprocessing for MUS extraction in \cite{DBLP:conf/tacas/BelovJM13}. We briefly review the framework here, and refer the reader to \cite{DBLP:journals/corr/abs-1207-1257,DBLP:conf/tacas/BelovJM13} for details. 

\vspace{-10pt}
\subsubsection{Labelled CNFs}

Assume a countable set of labels $Lbls$. A \defterm{labelled clause} (L-clause) is a tuple $\tuple{C,L}$, where $C$ is a clause, and $L$ is a finite (possibly empty) subset of $Lbls$. We denote the label-sets by superscripts, i.e. $C^L$ is the labelled clause $\tuple{C, L}$. A \defterm{labelled CNF (LCNF)} formula is a finite set of labelled clauses. 
For an LCNF formula $\Phi$~\footnote{We use capital Greek letters to distinguish LCNFs from CNFs.}, let 
$Cls(\Phi) = \bigcup_{C^L \in \Phi} \{ C \}$ be the \defterm{clause-set} of $\Phi$, and 
$Lbls(\Phi) =  \bigcup_{C^L \in \Phi} L$ be the \defterm{label-set} of $\Phi$.
LCNF satisfiability is defined in terms of the satisfiability of 
the clause-sets of an LCNF formula: $\Phi$ is satisfiable if and only if $Cls(\Phi)$ is satisfiable.
We will re-use the notation $\SAT$ (resp. $\UNSAT$) for the set of satisfiable (resp. unsatisfiable)
LCNF formulas\footnote{To avoid overly optimistic complexity results, we will tacitly 
assume that the sizes of label-sets of the clauses in LCNFs are polynomial in the 
number of the clauses}.  
However, the semantics of minimal unsatisfiability and maximal and maximum satisfiability of labelled CNFs 
are defined in terms of their label-sets via the concept of the \defterm{induced subformula}.
\begin{definition}[Induced subformula]\label{def:ind}
Let $\Phi$ be an LCNF formula, and let $M \subseteq Lbls(\Phi)$. The
subformula of $\Phi$ \defterm{induced by $M$} is the LCNF formula 
$\Phi|_M = \{ C^L \in \Phi\ |\ L \subseteq M \}$.
\end{definition}
In other words, $\Phi|_M$ consists of those labelled clauses of $\Phi$ whose 
label-sets are included in $M$, and so $Lbls(\Phi|_M) \subseteq M$, and
$Cls(\Phi|_M) \subseteq Cls(\Phi)$.
Alternatively, any clause that has at least one label outside of $M$ is removed 
from $\Phi$. Thus, it is convenient to talk about
the \defterm{removal} of a label from $\Phi$. Let $l \in Lbls(\Phi)$ be any label.
The LCNF formula $\Phi|_{M \setminus \{ l \}}$ is said to be obtained by 
the \defterm{removal of label $l$ from $\Phi$}. 

To the readers familiar with the assumption-based incremental SAT (c.f. \cite{DBLP:journals/entcs/EenS03}), it might be helpful to think of labels as selector variables attached to clauses of a CNF formula, taking into account the possibility of having multiple, or none at all, selectors for each clause\footnote{Furthermore, notice that clauses with multiple selectors show up exactly when resolution-based preprocessing is applied in the context of incremental SAT.}. Then an induced subformula $\Phi|_M$ is obtained by ``turning-on'' the selectors in $M$, and ``turning-off'' the selectors outside of $M$. An operation of removal of a label $l$ from $\Phi$ can be seen as an operation of ``turning-off'' the selector $l$. 
 
The concept of induced subformulas allows to adopt all notions related to satisfiability of subsets of CNF formulas to LCNF setting. For example, given an unsatisfiable LCNF $\Phi$, an unsatisfiable core of $\Phi$ is any set of labels $C \subseteq Lbls(\Phi)$ such that $\Phi|_C \in \UNSAT$. Note that the selectors that appear in the final conflict clause in the context of assumption-based incremental SAT constitute such a core. 
Furthermore, given an unsatisfiable LCNF $\Phi$, a set of labels $M \subseteq Lbls(\Phi)$ is an \defterm{MUS} of $\Phi$, if $(i)$ $\Phi|_M \in \UNSAT$, and $(ii)$ $\forall l \in M, \Phi|_{M \setminus \{ l \}} \in \SAT$. As with CNFs, the set of all MUSes of LCNF $\Phi$ is denoted by $\MUS(\Phi)$.
MSSes and MCSes of LCNF formulas can be defined in the similar manner. Specifically, for an unsatisfiable LCNF formula $\Phi$, a set of labels $R \subseteq Lbls(\Phi)$ is an \defterm{MCS} of $\Phi$, if $(i)$ $\Phi|_{Lbls(\Phi) \setminus R} \in \SAT$, and $(ii)$ $\forall l \in R$, $\Phi|_{(Lbls(\Phi) \setminus R) \cup \{ l \}} \in \UNSAT$. The set of all MCSes of $\Phi$ is denoted by $\MCS(\Phi)$.
It was shown in \cite{DBLP:journals/corr/abs-1207-1257} that the hitting-sets duality holds for LCNFs,~i.e. for any LCNF $\Phi$, $M \subseteq Lbls(\Phi)$ is an MUS of $\Phi$ if and only if $M$ is an irreducible hitting set of $\MCS(\Phi)$, and vice versa.

\begin{example}\label{ex:a}
Let $\Phi = \{ (\neg p)^{\emptyset}, (r)^{\emptyset}, (p \lor q)^{\{1\}}, (p \lor \neg q)^{\{1,2\}}, (p)^{\{2\}}, (\neg r)^{\{3\}} \}$. The label-set of a clause is given in the superscript, i.e. $Lbls = \N^+$ and  $Lbls(\Phi) = \{ 1, 2, 3 \}$. The subformula induced by the set $S = \{1\}$ is $\Phi|_S = \{ (\neg p)^{\emptyset}, (r)^{\emptyset}, (p \lor q)^{\{1\}} \}$. $S$ is an MSS of $\Phi$, as $\Phi|_S \in \SAT$ and both formulas $\Phi|_{\{1,2\}}$ and $\Phi|_{\{1,3\}}$
are unsatisfiable. $R = \{ 2, 3 \}$ is the corresponding MCS of $\Phi$.
\end{example}

To clarify the connection between LCNF and CNF formulas further, consider a CNF formula $F = \{ C_1, \dots, C_n \}$.
The LCNF formula $\Phi_{F}$ \defterm{associated with $F$} is constructed by labelling each clause $C_i \in F$ 
with  a \emph{unique, singleton} labelset $\{ i \}$, i.e. $\Phi_{F} = \{ C_i^{\{i\}}\ |\ C_i \in F \}$. 
Then, a removal of a label $i$ from $\Phi_{F}$ corresponds to a
removal of a clause $C_i$ from $F$, and so every MUS (resp. MSS/MCS) of $\Phi_{F}$ corresponds to an MUS 
(resp. MSS/MCS) of $F$ and vice versa.

The resolution rule for labelled clauses is defined as follows \cite{DBLP:conf/tacas/BelovJM13}: for two labelled clauses $(x \vee A)^{L_1}$ and $(\neg x \vee B)^{L_2}$, the \defterm{resolvent} $C_1^{L_1} \resolve_x C_2^{L_2}$ is the labelled clause $(A \vee B)^{L_1 \cup L_2}$. The definition is extended to two sets of labelled clauses $\Phi_{x}$ and $\Phi_{\neg x}$ that contain the literal $x$ and $\neg x$ resp., as with CNFs. Finally, a labelled clause $C_1^{L_1}$ is said to \defterm{subsume} $C_2^{L_2}$, in symbols $C_1^{L_1} \subset C_2^{L_2}$, if $C_1 \subset C_2$ and $L_1 \subseteq L_2$. Again, the two definitions become immediate if one thinks of labels as selector variables in the context of incremental SAT. 

\vspace{-10pt}
\subsubsection{Resolution and subsumption based preprocessing for LCNFs}

Resolution and subsumption based SAT preprocessing techniques discussed in Section~\ref{s:prelim} can be applied to LCNFs \cite{DBLP:conf/tacas/BelovJM13}, so long as the resolution rule and the definition of subsumption is taken to be as above. Specifically, define $\sve(\Phi,x) = \Phi \setminus (\Phi_x \cup \Phi_{\neg x}) \cup (\Phi_x \resolve_x \Phi_{\neg x})$. Then, an atomic operation of bounded variable elimination for LCNF $\Phi$ is defined as $\sbve(\Phi,x) = {\bf if\ } (|\sve(\Phi,x)| < |\Phi|) {\bf\ then\ } \sve(\Phi,x) {\bf\ else\ } \Phi$. The size of $\Phi$ is just the number of labelled clauses in it. A formula $\bve(\Phi)$ is obtained by applying $\sbve(\Phi,x)$ to all variables in $\Phi$.  
Similarly, for $C_1^{L_1}, C_2^{L_2} \in F$, define 
$\ssub(\Phi,C_1^{L_1},C_2^{L_2}) = {\bf if\ } (C_1^{L_1} \subset C_2^{L_2}) {\bf\ then\ }\Phi \setminus \{ C_2^{L_2} \} {\bf\ else\ } \Phi$. The formula $\sub(\Phi)$ is then obtained by applying $\ssub(\Phi,C_1^{L_1},C_2^{L_2})$ to all clauses of $\Phi$.
Finally, given two labelled clauses $C_1^{L_1} = (l \lor A)^{L_1}$ and $C_2^{L_2} = ({\neg l} \lor B)^{L_2}$ in $\Phi$, such that $A \subset B$ \emph{and} $L_1 \subseteq L_2$, the atomic step of self-subsuming resolution, $\sssr(\Phi,C_1^{L_1},C_2^{L_2})$, results in the formula $\Phi \setminus \{ C_2^{L_2} \} \cup \{ B^{L_2} \}$.
Notice that the operations $\sbve$ and $\sssr$ do not affect the set of \emph{labels} of the LCNF formula, however it might be the case that $\ssub$ removes some labels from it.


The soundness of the resolution and subsumption based preprocessing for LCNFs with respect to the computation of MUSes has been established in \cite{DBLP:conf/tacas/BelovJM13} (Theorem 1, Prop.~6 and 7). Specifically, given an LCNF $\Phi$, $\MUS(\sbve(\Phi,x)) \subseteq \MUS(\Phi)$, $\MUS(\ssub(\Phi,C_1^{L_1},C_2^{L_2})) \subseteq \MUS(\Phi)$, and $\MUS(\sssr(\Phi,C_1^{L_1},C_2^{L_2})) \subseteq \MUS(\Phi)$. In this paper we establish stronger statements that, by the hitting-sets duality for LCNFs \cite{DBLP:journals/corr/abs-1207-1257}, also imply that the set inclusions $\subseteq$ between the sets $\MUS(\circ)$ are set equalities.

\begin{proposition}\label{prop:lve}
For any LCNF formula $\Phi$ and variable $x$, $\MCS(\sbve(\Phi, x)) = \MCS(\Phi)$.
\end{proposition}
\begin{proof}
Assume that $\Phi' = \sve(\Phi, x)$ (i.e. the variable is actually eliminated), otherwise the claim is trivially true.

Let $L' = Lbls(\Phi')$, and $R'$ be an MCS of $\Phi'$, i.e. $R' \subseteq L'$, $\Phi'|_{L' \setminus R'} \in \SAT$, and $\forall l \in R'$, $\Phi'|_{(L' \setminus R') \cup \{l\}} \in \UNSAT$. Lemma~1 in \cite{DBLP:conf/tacas/BelovJM13} states that for any LCNF $\Phi$ and a set of labels $M$, $\sve(\Phi,x)|_{M} = \sve(\Phi|_{M},x)$,~i.e. the operations $\sve$ and $|_{M}$ commute. 
Thus, $\Phi'|_{L' \setminus R'} = \sve(\Phi,x)|_{L' \setminus R'} = \sve(\Phi|_{L' \setminus R'}, x)$. Since $\sve$ preserves satisfiability, and since $Lbls(\Phi) = Lbls(\Phi') = L'$, we conclude that $\Phi|_{Lbls(\Phi) \setminus R'} \in \SAT$. In the same way, we have $\forall l \in R'$, $\Phi|_{(Lbls(\Phi) \setminus R') \cup \{ l \}} \in \UNSAT$, i.e. $R$ is an MCS of $\Phi$. The opposite direction is shown by retracing the steps in reverse.\qed
\end{proof}

\begin{proposition}\label{prop:lsub}
For any LCNF formula $\Phi$, and any two clauses $C_1^{L_1},C_2^{L_2} \in \Phi$,\\ $\MCS(\ssub(\Phi,C_1^{L_1},C_2^{L_2})) = \MCS(\Phi)$.
\end{proposition}
\begin{proof}
Assume that $C_1^{L_1} \subset C_2^{L_2}$, and so $\Phi' = \ssub(\Phi,C_1^{L_1},C_2^{L_2}) = \Phi \setminus \{ C_2^{L_2} \}$. The proof is a bit more technical, due to the possibility of $Lbls(\Phi') \subset Lbls(\Phi)$. Let $M^* = Lbls(\Phi) \setminus Lbls(\Phi')$, that is, $M^*$ is the (possibly empty) set of labels that occur \emph{only} in the clause $C_2^{L_2}$.
We first establish a number of useful properties: let $M \subseteq Lbls(\Phi')$ (note that $M \cap M^* = \emptyset$).~Then,
$(p1)$ $\Phi|_{M \cup M^*} = \Phi'|_M \cup \{ C_2^{L_2} \}|_{M \cup M^*}$, and 
$(p2)$ if $L_2 \subseteq M \cup M^*$, then $\Phi|_{M \cup M^*} = \Phi'|_M \cup \{ C_2^{L_2} \}$, and, 
furthermore, $C_1^{L_1} \in \Phi'|_M$.

To prove $(p1)$ we note that since $\Phi = \Phi' \cup \{ C_2^{L_2} \}$, we have 
$\Phi|_{M \cup M^*} = \Phi'|_{M \cup M^*} \cup \{ C_2^{L_2} \}|_{M \cup M^*}$, and since none of the labels
from $M^*$ occur in $\Phi'$, we have $\Phi'|_{M \cup M^*} = \Phi'|_M$. To prove the first
part of $(p2)$ we use $(p1)$ together with the fact that since $L_2 \subseteq {M \cup M^*}$,
we have $\{ C_2^{L_2} \}|_{M \cup M^*} = \{ C_2^{L_2} \}$.
For the second part of $(p2)$ we note that since $L_1 \subseteq L_2$, and 
since $L_2 \subseteq M \cup M^*$, it must be that $C_1^{L_1} \in \Phi|_{M \cup M^*}$, and so,
by the first part of $(p2)$, in $\Phi'|_M$.

%
%
\medskip
We come back to the proof of the main claim. To show $\LMCS(\Phi') \subseteq \LMCS(\Phi)$, let $R$ be an 
LMCS of $\Phi'$, and let $M$ be the corresponding LMSS (i.e. $M = Lbls(\Phi') \setminus R$). 
We are going to show that $M \cup M^*$ is an LMSS of $\Phi$.

First, we establish that $\Phi|_{M \cup M^*} \in \SAT$. If $L_2 \not\subseteq M \cup M^*$, then by $(p1)$ we have 
$\Phi|_{M \cup M^*} = \Phi'|_M$, and since $\Phi'|_M \in \SAT$, we have $\Phi|_{M \cup M^*} \in \SAT$.
If, on the other hand, $L_2 \subseteq M \cup M^*$, then by $(p2)$ we have $\Phi|_{M \cup M^*} = \Phi'|_M \cup \{ C^L \}$, 
and that $C_1^{L_1} \in \Phi'|_M$. Then, since $C_1 \subset C_2$, any model
of $\Phi'|_M$ will also satisfy $C_2^{L_2}$, and since $\Phi'|_M \in \SAT$, we have $\Phi|_{M \cup M^*} \in \SAT$.

Now, let $M' = M \cup M^* \cup \{l\}$ for some $l \in R$. We need to show that 
$\Phi|_{M'} \in \UNSAT$. 
Let $M'' = M' \setminus M^*$. Note that since $M \cap M^* = \emptyset$, we have 
$M \subset M'' \subseteq Lbls(\Phi')$. Since $\Phi|_{M'} = \Phi|_{M'' \cup M^*}$, by $(p1)$ we have 
$\Phi|_{M'} = \Phi'|_{M''} \cup \{ C_2^{L_2} \}|_{M'}$. Furthermore, since $M$ is an LMSS of $\Phi'$, and 
$M \subset M'' \subseteq Lbls(\Phi')$, we have $\Phi'|_{M''} \in \UNSAT$, and so 
$\Phi|_{M'} \in \UNSAT$. 

We conclude that $M \cup M^*$ is an LMSS of $\Phi$, and since $R = Lbls(\Phi') \setminus M = Lbls(\Phi) \setminus (M \cup M^*)$
we conclude that $R$ is an LMCS of $\Phi$.

\medskip
For the opposite inclusion, let $R$ be an LMCS of $\Phi$. We first note that $R \cap M^* = \emptyset$, as otherwise $R$ cannot be an MCS of $\Phi'$. This is due to the fact that for any $M \subseteq Lbls(\Phi)$, if $\Phi|_M \in \SAT$ then $\Phi|_{M \cup M^*} \in \SAT$: since the labels from $M^*$ appear only in $C_2^{L_2}$, we have either $\Phi|_{M \cup M^*} = \Phi|_M$, or $\Phi|_{M \cup M^*} = \Phi|_M \cup \{ C_2^{L_2} \}$, and in the latter case, $L_2 \subseteq M \cup M^*$ and so $L_1 \subseteq M$, and so $C_1^{L_1} \in \Phi|_M$, and hence any model of $\Phi|_M$ satisfies $C_2^{L_2}$. 

Since we now have $R \cap M^* = \emptyset$, we have $Lbls(\Phi) \setminus R = M \uplus M^*$. Note that
$M \uplus M^*$ is an LMSS of $\Phi$. Furthermore, since $Lbls(\Phi') = Lbls(\Phi) \setminus M^*$, we have 
$Lbls(\Phi') \setminus R = M$. Thus, in order to prove that $R$ is an LMCS of $\Phi'$, it suffices to show that $M$ is an LMSS of $\Phi'$, given that $M \uplus M^*$ is an LMSS of $\Phi$. This is shown by retracing the steps of the first part in reverse.\qed
\end{proof}

\begin{proposition}\label{prop:lssr}
For any LCNF formula $\Phi$, and any two clauses $C_1^{L_1},C_2^{L_2} \in \Phi$,\\ $\MCS(\sssr(\Phi,C_1^{L_1},C_2^{L_2})) = \MCS(\Phi)$.
\end{proposition}
\begin{proof}
Assume that $C_1^{L_1} = (l \lor A)^{L_1}$ and $C_2^{L_2} = ({\neg l} \lor B)^{L_2}$ such that $A \subset B$ and $L_1 \subseteq L_2$, and so $\Phi' = \sssr(\Phi,C_1^{L_1},C_2^{L_2}) = \Phi \setminus \{ C_2^{L_2} \} \cup \{ B^{L_2} \}$. The claim is immediate from the fact that since $\Phi' \equiv \Phi$ and $Lbls(\Phi') = Lbls(\Phi)$, for any set of labels $M$, $\Phi'|_M \equiv \Phi|_M$. \qed
\end{proof}

To summarize, the three SAT preprocessing techniques discussed in this section, namely bounded variable elimination, subsumption elimination and self-subsuming resolution, preserve MCSes of LCNF formulas. Given that the MaxSAT problem for weighted CNFs can be cast as a problem of finding a minimum-cost MCS (cf.~Section~\ref{s:prelim}), we now define the MaxSAT problem for weighted LCNFs, and draw a connection between the two problems.
  
\vspace{-10pt}
\subsubsection{Maximum satisfiability for LCNFs}
Recall that the maximum satisfiability problem for a given weighted CNF formula $F = F^H \cup F^S$ can be seen as a problem of finding a minimum-cost set of soft clauses $R_{min}$ whose removal from $F$ makes $F$ satisfiable, i.e. a minimum-cost MCS of $F$. In LCNF framework we do not remove clause directly, but rather via labels associated with them. Thus, a clause labelled with an empty set of labels cannot be removed from an LCNF formula, and can play a role of a hard clause in a WCNF formula. By associating the weights to \emph{labels} of LCNF formula, we can arrive at a concept of a minimum-cost set of labels, and from here at the idea of the maximum satisfiability problem for LCNF formulas.

Thus, we now have \defterm{weighted labels} $(l, w)$, with $l \in Lbls$, and $w \in \N^+$ (note that there's no need for the special weight $\top$). A \defterm{cost} of a set $L$ of weighted labels is the sum of their weights. A \defterm{weighted LCNF formula} is a set of clauses labelled with weighted labels.
It is more convenient to define a MaxSAT solution for weighted LCNFs in terms of minimum-cost MCSes, rather that in terms of MaxSAT models. This is due to the fact that given an arbitrary assignment $\tau$ that satisfies all clauses labelled with $\emptyset$, the definition of a ``set of labels falsified by $\tau$'' is not immediate, since in principle a clause might be labelled with more than one label, and, from the MaxSAT point of view, we do not want to remove more labels than necessary.

\begin{definition}[MaxSAT solution for weighted LCNF]\label{def:lmaxsat}
Let $\Phi$ be a weighted LCNF formula with $\Phi|_{\emptyset} \in \SAT$. An assignment $\tau$ is a \emph{MaxSAT solution} of $\Phi$ if $\tau$ is a model of the formula $\Phi|_{Lbls(\Phi) \setminus R_{min}}$ for some minimum-cost MCS $R_{min}$ of $\Phi$. The cost of $\tau$ is the cost of $R_{min}$.
\end{definition}
In other words, a MaxSAT solution $\tau$ for a weighted LCNF maximizes the cost 
of a set $S \subseteq Lbls(\Phi)$, subject to $\tau$ satisfying $\Phi|_S$, and the cost of $\tau$ is the cost of the set $R = Lbls(\Phi) \setminus S$.

Let $F = F^H \cup F^S$ be a weighted CNF formula. The weighted LCNF formula $\Phi_{F}$ \defterm{associated with $F$} is constructed similary to the case of plain CNFs: assuming that $F^S = \{ C_1, \dots, C_n \}$, we will use $\{1,\dots,n\}$ to label the soft clauses, so that a clause $C_i$ gets a unique, singleton labelset $\{ i \}$, hard clauses will be labelled with $\emptyset$, and the weight of a label $i$ will be set to be the weight of the soft clause $C_i$.  
Formally, $Lbls(\Phi) = \{1,\dots,|F^S|\} \subset \N^+$, $\Phi_{F} = (\cup_{C \in F^H} \{ C^{\emptyset} \}) \cup (\cup_{C_i \in F^S} \{ C_i^{\{i\}} \}$, and $\forall i \in Lbls(\Phi), w(i) = w(C_i)$.

Let $\Phi_F$ be the weighted LCNF formula associated a weighted CNF $F$. Clearly, every MaxSAT solution of $\Phi_F$ is a MaxSAT solution of $F$, and vice versa. In the previous subsection we showed that the resolution and the subsumption elimination based preprocessing techniques preserve the MCSes of $\Phi_F$. We will show  shortly that this leads to the conclusion that the techniques can be applied soundly to $\Phi_F$, and so, assuming the availability of a method for solving MaxSAT problem for $\Phi_F$ (Section~\ref{s:alg}), this allows to use preprocessing, albeit indirectly, for solving MaxSAT problem for $F$.

\vspace{-10pt}
\subsubsection{Preprocessing and MaxSAT for LCNFs}

\begin{theorem}\label{th:rs}
For weighted LCNF formulas, the atomic operations of bounded variable elimination ($\sbve$), subsumption elimination ($\ssub$), and self-subsuming resolution ($\sssr$) sound for MaxSAT.
\end{theorem}
\begin{proof}
Let $\Phi$ be a weighted LCNF formula. Assume that for some variable $x$, $\Phi' = \sbve(\Phi,x)$, and let $\tau'$ be a MaxSAT solution of $\Phi'$. Thus, for some minimum-cost MCS $R_{min}$ of $\Phi'$, $\tau'$ is a model of $\Phi'|_{Lbls(\Phi') \setminus R_{min}}$. By Proposition~\ref{prop:lve}, $R_{min}$ is a minimum-cost MCS of $\Phi$. If $x$ was eliminated, $\tau'$ can be transformed in linear time to a model $\tau$ of $\Phi|_{Lbls(\Phi) \setminus R_{min}}$ by assigning the truth-value to $x$ (cf.~\cite{DBLP:conf/sat/JarvisaloB10}). We conclude that $\sbve$ is sound for LCNF MaxSAT.

For $\ssub$ and $\sssr$ no reconstruction is required, since the techniques preserve equivalence. The claim of the theorem follows directly from Propositions~\ref{prop:lsub} and \ref{prop:lssr}. \qed
\end{proof}

\medskip

\noindent
To conclude this section, lets us summarize the SAT preprocessing ``pipeline'' for solving the MaxSAT problem for weighted CNFs. Given a WCNF formula $F$, first apply any MUS-preserving (and so, monotone) clause-elimination technique, such as BCE, to obtain the formula $F'$. Then, construct an LCNF formula $\Phi|_{F'}$ associated with $F'$, and apply BVE, subsumption elimination and SSR, possibly in an interleaved manner, to $\Phi|_{F'}$ to obtain $\Phi'$. Solve the MaxSAT problem for $\Phi'$, and reconstruct the solution to the MaxSAT problem of the original formula $F$ --- Theorems~\ref{th:mon} and \ref{th:rs} show that it can be done feasibly. The only missing piece is how to solve MaxSAT problem for LCNF formulas --- this is the subject of the next section.

We have to point out that the resolution and the subsumption elimination preprocessing techniques in the LCNF framework are not without their limitations. For $\bve$ the label-sets of clauses grow, which may have a negative impact on the performance of SAT solvers if LCNF algorithms are implemented incrementally. Also, two clauses $C^{L_1}$ and $C^{L_2}$ are treated as two different clauses if $L_1 \ne L_2$, while without labels they would be collapsed into one, and thus more variables might be eliminated. Nevertheless, when many hard (i.e. labelled with $\emptyset$) clauses are present, this negative effect is dampened. For subsumption elimination the rule $L_1 \subseteq L_2$ is quite restrictive. In particular, it blocks subsumption completely in the plain MaxSAT setting (though, as we already saw, unrestricted subsumption is dangerous for MaxSAT). However, in partial MaxSAT setting it does enable the removal of any clause (hard or soft) subsumed by a hard clause. 
In Section~\ref{s:exp}, we demonstrate that the techniques do lead to performance improvements in practice.

\section{Solving MaxSAT problem for LCNFs}\label{s:alg}

In this section we propose two methods for solving MaxSAT problem for weighted LCNFs.
Both methods rely on the connection between the labels in LCNFs and the selector variables.

\vspace{-10pt}
\subsection{Reduction to weighted partial MaxSAT}\label{s:map}

The idea of this method is to encode a given weighted LCNF formula $\Phi$ as an WCNF formula $F_{\Phi}$, mapping the labels of $\Phi$ to soft clauses in such a way that a removal of soft clause from $F_{\Phi}$ would emulate the operation of a removal of a corresponding label from $\Phi$. This is done in the following way: for each $l_i \in Lbls(\Phi)$, create a new variable $a_i$. Then, for each labelled clause $C^L$ create a \emph{hard} clause $C \lor \bigvee_{l_i \in L} (\neg a_i)$. Finally, for each $l_i \in Lbls(\Phi)$, create a \emph{soft} clause $(a_i)$ with a weight equal to the weight of the label $l_i$. 

\begin{example}
Let $\Phi = \{ (\neg p)^{\emptyset}, (r)^{\emptyset}, (p \lor q)^{\{1\}}, (p \lor \neg q)^{\{1,2\}}, (p)^{\{2\}}, (\neg r)^{\{3\}} \}$, and assume that the weights of all labels are 1. Then, $F_{\Phi} = \{ (\neg p, \top), (r, \top), (\neg a_1 \lor p \lor q, \top), (\neg a_1 \lor \neg a_2 \lor p \lor \neg q, \top), (\neg a_2 \lor p, \top), (\neg a_3 \lor \neg r, \top), (a_1, 1), (a_2, 1), (a_3, 1) \}$. Then, removal of $(a_2, 1)$ from the $F_{\Phi}$ leaves $\neg a_2$ pure, and so is equivalent to the removal of all hard clauses clauses that contain $a_2$, which in turn is equivalent to the removal of the label 2 from $\Phi$.
\end{example}

It is then not difficult to see that any MaxSAT solution of $F_{\Phi}$ is a MaxSAT solution of $\Phi$, and vice versa. The advantage of the indirect method is that any off-the-shelf MaxSAT solver can be turned into a MaxSAT solver for LCNFs. However, it also creates a level of indirection between the selector variables and the clauses they are used in. In our preliminary experiments the indirect method did not perform well.

\vspace{-10pt}
\subsection{Direct computation}\label{s:direct}

Core-guided MaxSAT algorithms are among the strongest algorithms for industrially-relevant MaxSAT problems. These algorithms iteratively invoke a SAT solver, and for each unsatisfiable outcome, \defterm{relax} the clauses that appear in the unsatisfiable core returned by the SAT solver. A clause $C_i$ is \defterm{relaxed} by adding a literal $r_i$ to $C_i$ for a fresh \defterm{relaxation variable} $r_i$. Subsequently, a cardinality or a pseudo-Boolean constraint over the relaxation variables $r_i$ is added to the set of the hard clauses of the formula. The exact mechanism is algorithm-dependent --- we refer the reader to the recent survey of core-guided MaxSAT algorithms in \cite{mhlpms-cj13}.

The key idea that enables to adapt core-guided MaxSAT algorithms to the LCNF setting is that the ``first-class citizen'' in the context of LCNF is not a clause, but rather a \emph{label}. In particular, the unsatisfiable core returned by a SAT solver has to be expressed in terms of the labels of the clauses that appear in the core. Furthermore, in the LCNF setting, it is the labels that get relaxed, and not the clauses directly. That is, when a label $l_i$ is relaxed due to the fact that it appeared in an unsatisfiable core, the relaxation variable $r_i$ is added to all clauses whose labelsets include $l_i$.

\begin{figure}[t]
  \begin{minipage}{0.49\textwidth}\vspace{-15pt}
    {\small \input{algs/fumal}}
    \caption{Fu and Malik algorithm for partial MaxSAT \cite{DBLP:conf/sat/FuM06}}\label{alg:fumal}
  \end{minipage}
  \begin{minipage}{0.02\textwidth}
    \hfill
  \end{minipage}
  \begin{minipage}{0.49\textwidth}
    {\small \input{algs/lfumal}}
    \caption{(Unweighted) LCNF version of Fu and Malik algorithm}\label{alg:lfumal}
  \end{minipage}
  \vspace{-15pt}
\end{figure}

To illustrate the idea consider the pseudocode of a core-guided algorithm for solving partial MaxSAT problem due to Fu and Malik \cite{DBLP:conf/sat/FuM06}, presented in Figure~\ref{alg:fumal}. And, contrast it with the (unweighted) LCNF-based version of the algorithm, presented in Figure~\ref{alg:lfumal}. 
The original algorithm invokes a SAT solver on the, initially input, formula $F$ until the formula is satisfiable. For each unsatisfiable outcome, the soft clauses that appear in the unsatisfiable core $Core$ (assumed to be returned by the SAT solver) are relaxed (lines~5-7), and the CNF representation of the $equals1$ constraint on the sum of relaxation variables is added to the set of the hard clauses of $F$. 
The LCNF version of the algorithm proceeds similarly. The only two differences are as follows. When the LCNF formula $\Phi$ is unsatisfiable, the unsatisfiable core has to be expressed in terms of the labels, rather than clauses. That is, the algorithm expects to receive a set $L_{core} \subseteq Lbls(\Phi)$ such that $\Phi|_{L_{core}} \in \UNSAT$. Some of the possible ways to obtain such a set of \defterm{core labels} are described shortly. The second difference is that a fresh relaxation variable $r_i$ is associated with each core label $l_i$, rather than with each clause as in the original algorithm. Each core label $l_i$ is relaxed by replacing each clause $C^L$ such that $l_i \in L$ with $(r_i \lor C)^L$ (lines~7-8). Note that in principle $C^L$ may include more than one core label, and so may receive more than relaxation variable in each iteration of the algorithm. The nested loop on lines~5-8 of the algorithm can be replaced by a single loop iterating over all clauses $C^L$ such that $L \cap L_{core} \ne \emptyset$. Finally, the clauses of the CNF representation of the $equals1$ constraint are labelled with $\emptyset$, and added to $\Phi$.

One of the possible ways to obtain the set of core labels is to use a standard core-producing SAT solver. One can use either a proof-tracing SAT solver, such as PicoSAT \cite{bierre-jsat08}, that extracts the core from the trace, or an assumption-based SAT solver, that extracts the core from the final conflict clause. Then, to check the satisfiability of $\Phi$, the clause-set $Cls(\Phi)$ of $\Phi$ is passed to a SAT solver, and given an unsatisfiable core $Core \subseteq Cls(\Phi)$, the set of core labels is obtained by taking a union of the labels of clauses that appear in $Core$. Regardless of the type of the SAT solver, the solver is invoked in \emph{non-incremental} fashion, i.e. on each iteration of the main loop a new instance of a SAT solver is created, and the clauses $Cls(\Phi)$ are passed to it. It is worth to point out that the majority of SAT-based MaxSAT solvers use SAT solvers in such non-incremental fashion. Also, it is commonly accepted that proof-tracing SAT solvers are superior to the assumption-based in the MaxSAT setting, since a large number of assumption literals tend to slow down SAT solving, while, at the same time, the incremental features of assumption-based solvers are not used.
 
An alternative to the non-incremental use of SAT solvers in our setting is to take advantage of the incremental features of the assumption-based SAT solvers. While we already explained that labels in LCNFs can be seen naturally as selectors in the assumption-based incremental SAT, the tricky issue is to emulate the operation of relaxing a clause, i.e. adding one or more relaxation variables to it. The only option in the incremental SAT setting is to ``remove'' the original clause by adding a unit clause $(\neg s)$ to the SAT solver for some selector literal $\neg s$, and add a relaxed version of the clause instead. The key observation here is that since the labels are already represented by selector variables, we can use \emph{these} selector variables to both to remove clauses and to keep track of the core labels. For this, each label $l_i \in Lbls(\Phi)$ is associated with a \emph{sequence} of selector variables $a_i^0, a_i^1, a_i^2, \dots$. At the beginning, just like in the reduction described in Section~\ref{s:map}, for each $C^L$ we load a clause $C' = C \lor \bigvee_{l_i \in L} (\neg a_i^0)$ into the SAT solver, and solve under assumptions $\{a_1^0, a_2^0, \dots\}$. The selectors that appear in the final conflict clause of the SAT solver will map to the set of the core labels $L_{core}$. Assume now that a label $l_c \in L$ is a core label, i.e. the selector $a_c^0$ was in the final conflict clause. And, for simplicity, assume that $l_c$ is the only core label in $L$. Now, to emulate the relaxation of the clause $C'$, we first add a unit clause $(\neg a_c^0)$ to the SAT solver to ``remove'' $C'$, and then add a clause $C'' = (C' \setminus \{ \neg a_c^0 \}) \cup \{ r, \neg a_c^1 \}$, where $r$ is the relaxation variable associated with $l_c$ in this iteration, and $a_c^1$ is a ``new version'' of a selector variable for $l_c$. If on some iteration $a_c^1$ appears in the final conflict clause, we will know that $l_c$ is a core label that needs to be relaxed, add $(\neg a_c^1)$ to the SAT solver, and create yet another version $a_c^2$ of a selector variable for the label $l_c$. For MaxSAT algorithms that relax each clause at most once (e.g. WMSU3 and BCD2, cf. \cite{mhlpms-cj13}), we only need two versions of selectors for each label.

Note that since, as explained in Section~\ref{s:premax}, MaxSAT problem for WCNF $F$ can be recast as a MaxSAT problem for the associated LCNF $\Phi_F$, the incremental-SAT based MaxSAT algorithms for LCNFs can be seen as incremental-SAT based MaxSAT algorithm for WCNFs --- to our knowledge such algorithms have not been previously described in the literature. 
The main advantage of using the SAT solver incrementally, beside the saving from re-loading the whole formula in each iteration of a MaxSAT algorithm, is in the possible reuse of the learned clauses between the iterations. While many of the clauses learned from the soft clauses will not be reused (since they would also need to be relaxed, otherwise), the clauses learned from the hard clauses will. In our experiments (see next section) we did observe gains from incrementality on instances of weighted partial MaxSAT problem.


\section{Experimental Evaluation}\label{s:exp}

To evaluate the ideas discussed in this paper empirically, we implemented an LCNF-based version of the MaxSAT algorithm WMSU1 \cite{DBLP:conf/sat/FuM06,DBLP:conf/sat/AnsoteguiBL09,DBLP:conf/sat/ManquinhoSP09}, which is an extension of Fu and Malik's algorithm discussed in Section~\ref{s:direct} to the weighted partial MaxSAT case. Note that none of the important optimizations discussed in \cite{DBLP:conf/sat/ManquinhoSP09} were employed. The algorithm was implemented in both the non-incremental and the incremental settings, and was evaluated on
the set of industrial benchmarks from the MaxSAT Evaluation
2013\footnote{http://maxsat.ia.udl.cat/}, a total of 1079 instances.
The experiments were performed on an HPC cluster, 
with quad-core Intel Xeon E5450 3 GHz nodes with 32 GB of memory.  
All tools were run with a timeout of 1800 seconds and a memory limit of 4 GB
per input instance.

In the experiments PicoSAT~\cite{bierre-jsat08} and Lingeling~\cite{lingeling} were 
used as the underlying SAT solvers.
For (pure) MaxSAT benchmarks, we used PicoSAT (v.~935),
while for partial and weighted partial MaxSAT instances we used PicoSAT (v.~954) --- 
the difference between versions is due to better performance in the preliminary experiments.
Both incremental (P) and non-incremental proof-tracing (P\_NI) settings for PicoSAT were
tested. For Lingeling (v.~ala) the incremental mode (L) was tested.

For the preprocessing, we implemented our own version of
Blocked Clause Elimination (BCE), while for Resolution and Subsumption (RS) both
SatElite~\cite{een-sat05} and Lingeling~\cite{lingeling} as a preprocessor were
used.
We have included in the experiments WMSU1 algorithm from
MSUnCore~\cite{DBLP:conf/sat/ManquinhoSP09} in order to establish a reasonable baseline.
  
\begin{figure}[t]
  \begin{tabular}{cc}
    \includegraphics[width=0.5\textwidth]{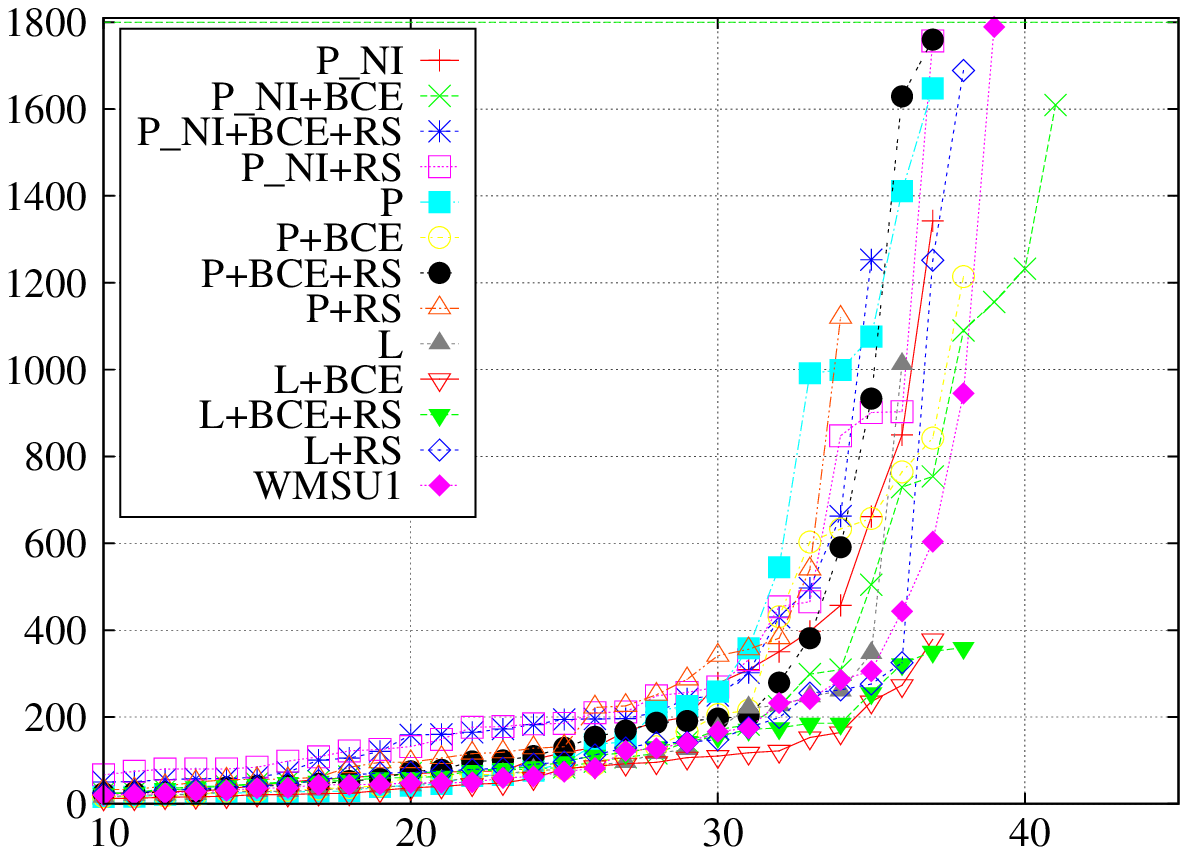}
	&
	\includegraphics[width=0.5\textwidth]{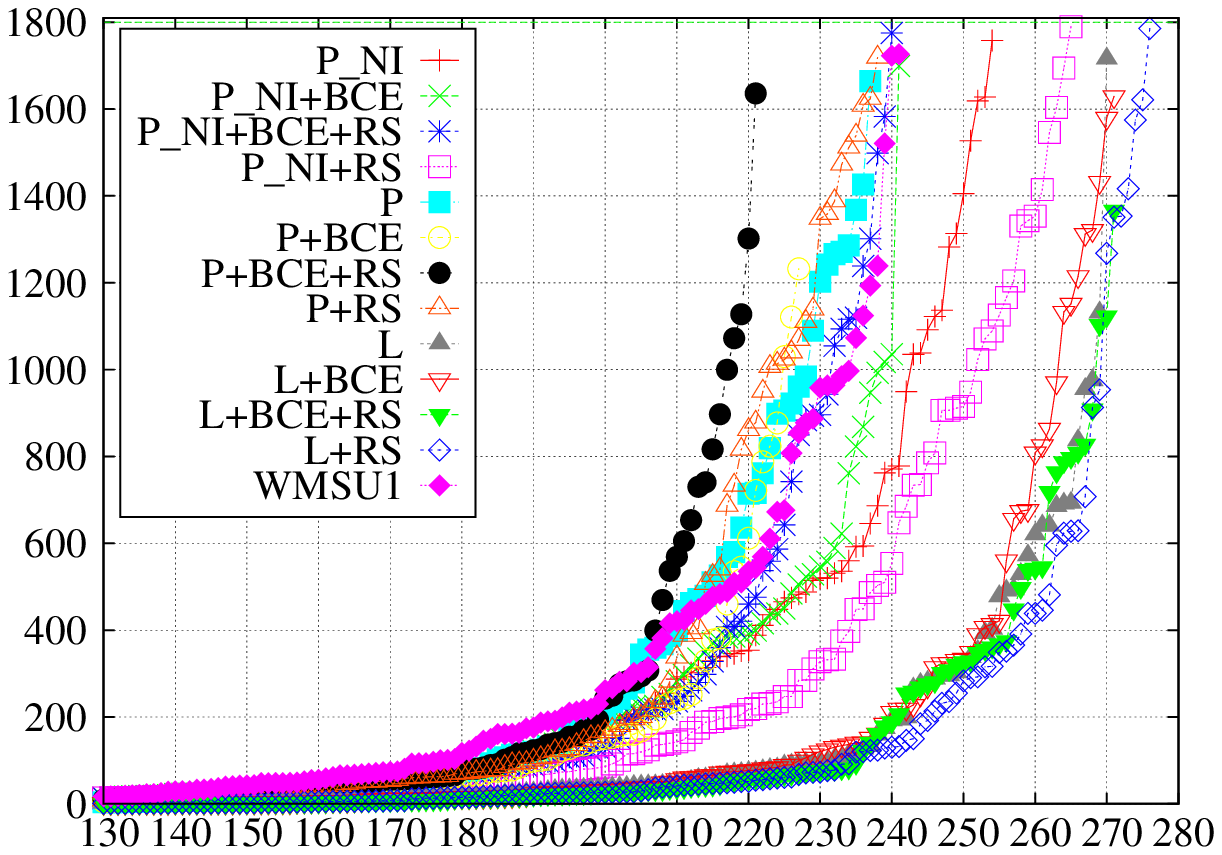} 
	\\
	(a) MaxSAT 
	&
	(b) Partial MaxSAT 
	\\
	\includegraphics[width=0.5\textwidth]{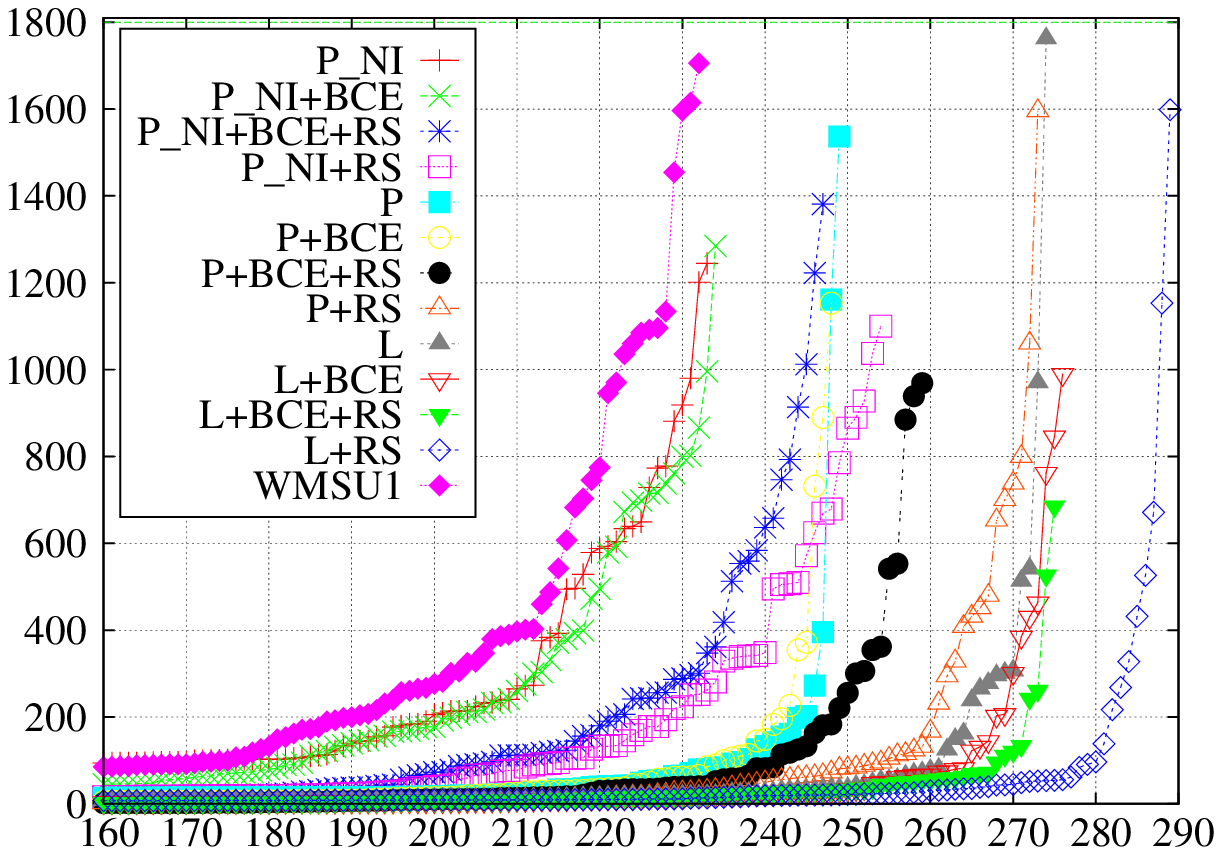} 
	&
 	\includegraphics[width=0.5\textwidth]{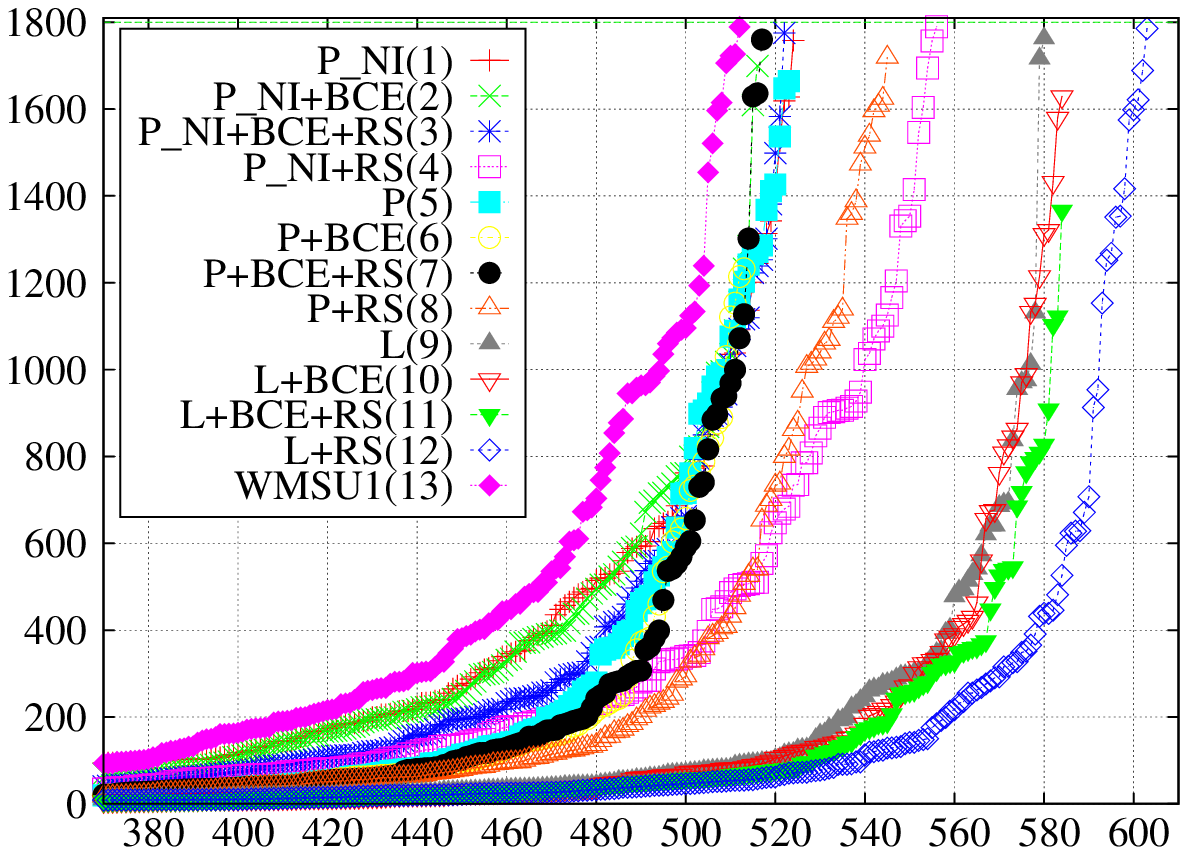} 
 	\\
 	(c) Weighted Partial MaxSAT 
 	&
 	(d) All
 	\end{tabular}
 	\caption{Cactus plots for the different categories.}
 	\label{fig:cactus}
        \vspace{-20pt}
\end{figure}

Figure~\ref{fig:cactus} shows the results for different classes of
industrial \mxsat instances, while Table~\ref{tbl:solved} complements it by
showing the number of solved instances by each configuration/solver, and the
average CPU time taken on the solved instances.
From the figure and the table, the following conclusions can be drawn. 
%
%
First, we note that the resolution and subsumption elimination based preprocessing (RS) is, in general, 
quite effective. In fact, for each of the solvers, within the same solver, the configuration
that outperforms all others is RS, except for plain \mxsat instances with PicoSAT.
Also L$+$RS solves the highest number of instances overall, as revealed in
Figure~\ref{fig:cactus}~(d). 
Regarding the blocked clause elimination (BCE), the technique is effective for plain \mxsat
instances, however not for other classes of instances.  
Notice that the combination of BCE$+$RS never improves over the
best of the techniques considered separately, being only equal with
Lingeling for (pure) \mxsat instances.

\begin{table}[t]
  \begin{center}
	\begin{tabular}{|l|c|c|c|c|c|c|c|c|}
\hline
&\multicolumn{2}{c|}{All}
&\multicolumn{2}{c|}{MaxSAT}
&\multicolumn{2}{c|}{Partial MaxSAT}
&\multicolumn{2}{c|}{Weighted Partial MaxSAT}
\\
&\#Sol.&A.CPU&\#Sol.&A.CPU&\#Sol.&A.CPU&\#Sol.&A.CPU\\
\hline
Instances&1079& &55& &627& &397&\\
\hline
P\_NI&524&144.29&37&172.76&254&152.04&233&131.32 \\
P\_NI$+$BCE&516&115.84&{\bf 41}&237.58&241&105.02&234&105.65 \\ 
P\_NI$+$BCE$+$RS&522&103.08&35&177.37&240&120.70&247&75.42 \\
P\_NI$+$RS&556&124.48&37&246.68&265&154.84&254&75.00 \\
\hline
P&523&91.81&37&236.26&237&132.83&249&31.31 \\
P$+$BCE&513&57.70&38&180.22&227&70.08&248&27.60 \\
P$+$BCE+RS&517&67.61&37&209.48&221&85.36&259&32.19 \\
P$+$RS&545&93.71&34&151.77&238&146.93&273&40.08 \\
\hline
L&580&55.93&36&101.92&270&75.45&274&30.64 \\
L$+$BCE&584&60.84&37&67.88&271&95.89&276&25.49 \\
L$+$BCE$+$RS&584&48.03&38&96.02&271&73.90&275&15.90 \\
L$+$RS&{\bf 603}&65.26&38&161.71&{\bf 276}&91.15&{\bf 289}&27.85 \\
\hline
WMSU1&512&157.68&39&165.64&241&149.01&232&165.35 \\
\hline
\end{tabular}
  \end{center}
  \caption{Table of solved instances and average CPU times}
  \label{tbl:solved}
  \vspace{-20pt}
\end{table}

Somewhat surprisingly, our results suggest that, in contrast with
standard practice (i.e.\ most \mxsat solvers are based on
non-incremental SAT), the incremental SAT solving can be effective for
some classes of \mxsat instances.
Namely for Weighted Partial \mxsat instances, where for example PicoSAT
incremental (P) solves 16 more instances than PicoSAT non-incremental (P\_NI)
with a much lower average CPU time on the solved instances.

Finally, comparing the underlying SAT solvers used, it can be seen that in our 
experiments Lingeling performs significantly better than PicoSAT, which, as our
additional experiments suggest, is in turn is much better SAT solver than 
Minisat \cite{DBLP:conf/sat/EenS03}, for MaxSAT problems.


\section{Conclusion}\label{s:conc}

In this paper we investigate the issue of sound application of SAT preprocessing techniques for solving the MaxSAT problem. To our knowledge, this is the first work that addresses this question directly. 
We showed that monotone clause elimination procedures, such as BCE, can be applied soundly on the input formula. We also showed that the resolution and subsumption elimination based techniques can be applied, although indirectly, through the labelled-CNF framework. 
Our experimental results suggest that BCE can be effective on (plain) MaxSAT problems, and that the LCNF-based resolution and subsumption elimination leads to performance boost in partial and weighted partial MaxSAT setting.
Additionally, we touched on an issue of the incremental use of assumption-based SAT solvers in the MaxSAT setting, and showed encouraging results on weighted partial MaxSAT problems.
In the future work we intend to investigate issues related to the sound application of additional SAT preprocessing techniques.

\medskip
\noindent{\bf Acknowledgements} We thank the anonymous referees for their comments and suggestions.

\bibliography{paper}
\bibliographystyle{plain}


\end{document}